\documentclass[conference,10pt]{IEEEtran}
\usepackage{amsmath,amsfonts, amssymb}
\usepackage{amsmath}
\usepackage{mathrsfs}

\newtheorem{theorem}{Theorem}[section]

\usepackage{bm}
\newenvironment{proof}{\textit{Proof.}}{\hfill$\square$}
\usepackage{chngcntr}
\usepackage{apptools}
\AtAppendix{\counterwithin{lemma}{section}}

\DeclareMathOperator*{\argmin}{argmin}
\usepackage{enumitem}

\newcommand{\at}[2][]{#1|_{#2}}
\usepackage{verbatim}
\usepackage{lipsum}
\usepackage{float}
\usepackage{multicol}
\usepackage{graphicx}

\begin{document}
\title{A Unified View Between Tensor Hypergraph Neural Networks And Signal Denoising}

\vspace{-0.2cm}
\author{{\normalsize Fuli~Wang, Karelia~Pena-Pena, Wei~Qian, and~Gonzalo~R.~Arce}\\
{\normalsize University of Delaware, Newark, DE, USA}\\
{\normalsize Email: \{fuliwang, kareliap, weiqian, arce\}@udel.edu}}

\maketitle

\begin{abstract}
Hypergraph Neural networks (HyperGNNs) and hypergraph signal denoising (HyperGSD) are two fundamental topics in higher-order network modeling. Understanding the connection between these two domains is particularly useful for designing novel HyperGNNs from a HyperGSD perspective, and vice versa. In particular, the tensor-hypergraph convolutional network (T-HGCN) has emerged as a powerful architecture for preserving higher-order interactions on hypergraphs, and this work shows an equivalence relation between a HyperGSD problem and the T-HGCN. Inspired by this intriguing result, we further design a tensor-hypergraph iterative network (T-HGIN) based on the HyperGSD problem, which takes advantage of a multi-step updating scheme in every single layer. Numerical experiments are conducted to show the promising applications of the proposed T-HGIN approach.
\end{abstract}
\begin{IEEEkeywords}
Hypergraph Neural Network, Hypergraph Signal Denoising, Hypergraph Tensor.
\end{IEEEkeywords}

\IEEEpeerreviewmaketitle

\section{Introduction}
\vspace{-0.3cm}
Hypergraphs are ubiquitous in real-world applications for representing interacting entities. Potential examples include biochemical reactions that often involve more than two interactive proteins~\cite{biomedical}, recommendation systems that contain more than two items in a shopping activity~\cite{nextitem}, and traffic flows that can be determined by more than two locations~\cite{zheng2020gman}. In a hypergraph, entities are described as vertices/nodes, and multiple connected nodes form a hyperedge as shown in Fig.~\ref{fig: intro} (b, c) of a hypergraph example.

A hypergraph $\mathcal{G}$ is defined as a pair of two sets $\mathcal{G} = (\mathcal{V}, \mathcal{E})$, where $\mathcal{V} = \{v_1, v_2, ..., v_N\}$ denotes the set of $N$ nodes and $\mathcal{E} = \{e_1, e_2, ..., e_K\}$ is the set of $K$ hyperedges whose elements $e_k$ ($k = 1, 2,...,K$) are nonempty subsets of $\mathcal{V}$. The maximum cardinality of edges, or $m.c.e(\mathcal{G})$, is denoted by $M$, which defines the order of a hypergraph. Apart from the hypergraph structure, there are also features $\mathbf{x}_v \in \mathbb{R}^D$ associated with each node $v\in\mathcal V$, which are used as row vectors to construct the feature matrix $\mathbf{X}\in \mathbb{R}^{N\times D}$ of a hypergraph. From a hypergraph signal processing perspective, since the feature matrix $\mathbf{X}$ can be viewed as a $D$-dimensional signal over each node, we use the words ``feature" and ``signal" interchangeably throughout the paper.

\begin{figure}[htbp]
\centerline{\includegraphics[width=0.45\textwidth]{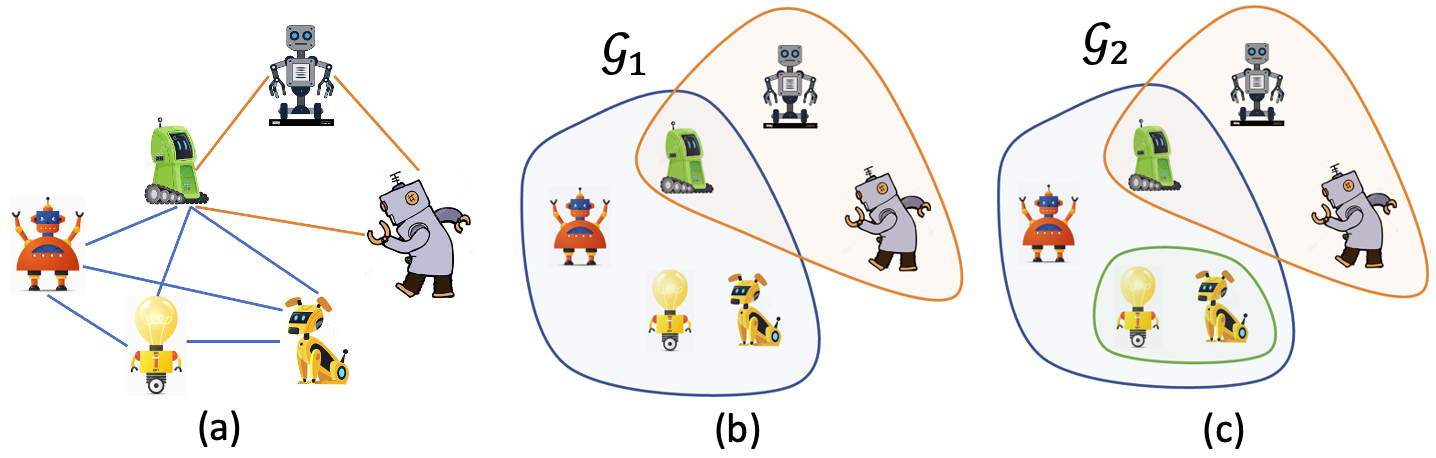}}
\vspace{-0.2cm}
\caption{Robot collaboration network represented by (a) a simple graph and (b) a hypergraph $\mathcal{G}_1$ and (c) another hypergraph $\mathcal{G}_2$. In (a), each cooperation relationship is denoted by a line connecting exactly two entities; whereas in (b) and (c), each hyperedge denoted by a colored ellipse represents multi-robot cooperation.}
\label{fig: intro}
\vspace{-0.5cm}
\end{figure}
Given the hypergraph structure $\mathcal{G}$ and the associated feature matrix $\mathbf{X}$, hypergraph neural networks (HyperGNNs) are built through two operations: 1) signal transformation and 2) signal shifting to leverage higher-order information. Specifically, if a HyperGNN is defined in a matrix setting, these two steps can be written as follows:
\begin{align}\label{eq:T-MPHN}
 \begin{cases}
        &\text{Signal transformation: } \mathbf{X}' = \phi_{trans}(\mathbf{X}; \mathcal{W}); \\
        & \text{Signal shifting: } \mathbf{Y} = \phi_{shift}(\mathbf{X}', \mathcal{G});
 \end{cases}
\end{align}
where $\mathbf{X}'$ is the transformed signal in a desired hidden dimension $D'$ and $\mathbf{Y}$ represents the linear combination of signals at the neighbors of each node according to the hypergraph structure $\mathcal{G}$. While here the variables are denoted by matrices, in fact, a tensor paradigm provides significant advantages~\cite{hypergraph_review} as will be introduced later, and thus will be at the core of this paper context. The signal transformation function $\phi_{trans}$, is parameterized by a learnable weight $\mathcal{W}$ and is generally constructed by multi-layer perceptrons (MLPs). As a result, the variation of HyperGNNs mainly lies in the signal-shifting step. To make use of the hypergraph structure in the signal-shifting step, an appropriate hypergraph algebraic descriptor is required. Prior efforts on HyperGNNs primarily focus on matrix representations of hypergraphs with possible information loss~\cite{hypergraph_review,wan2021principled}. Consider one of the most common hypergraph matrix representations, the adjacency matrix of the clique-expanded hypergraph used in~\cite{HGNN, attention}, which constructs pair-wise connections between any two nodes that are within the same hyperedge, thus only providing a non-injective mapping. As shown in Fig~\ref{fig: intro}, hypergraphs (b) $\mathcal{G}_1$ and (c) $\mathcal{G}_2$  have the same pairwise connections as the simple graph of Fig.~\ref{fig: intro} (a). 

Recently, a tensor-based HyperGNN framework T-HyperGNN~\cite{t-hypergnn} has been proposed to address potential information loss in matrix-based HyperGNNs. Specifically, the T-HyperGNN formulates tensor-hypergraph convolutional network (T-HGCN) via tensor-tensor multiplications (t-products)~\cite{order-p_tensor}, which fully exploits higher-order features carried by a hypergraph. Interestingly, we find that the hypergraph signal shifting in T-HGCN is equivalent to a one-step gradient descent of solving a hypergraph signal denoising (HyperGSD) problem (to be shown in Sec.~\ref{sec:equivalency}). Nevertheless, updating the gradient in one step per HyperGNN layer might be sub-optimal: For the two steps of HyperGNNs, only the signal shifting step corresponds to the gradient descent update. If we simply stack many layers of T-HGCN to perform  multi-step gradient descent as shown in Fig.~\ref{fig:modelcompare}(a), the number of learnable parameters will unnecessarily increase. More importantly, numerous sequential transformations of the hypergraph signals could cause indistinguishable features across all nodes, leading to the well-known over-smoothing problem~\cite{pagerank}. To overcome these issues, we propose an iterative $K$-step gradient descent procedure to solve the underlying HyperGSD problem, and further cast this procedure to formulate the novel {\bf Tensor-hypergraph iterative network (T-HGIN)}, which combines the $K$-step updating process (signal shifting) in just a single layer as shown in Fig.~\ref{fig:modelcompare}(b).
Additionally, T-HGIN leverages the initial input (with weight $\alpha$) and the current output (with weight $1-\alpha$) at each shifting step, performing a skip-connection operation that avoids over-smoothing.
\begin{figure}[htbp]
\vspace{-0.3cm}
\centerline{\includegraphics[width=0.42\textwidth]{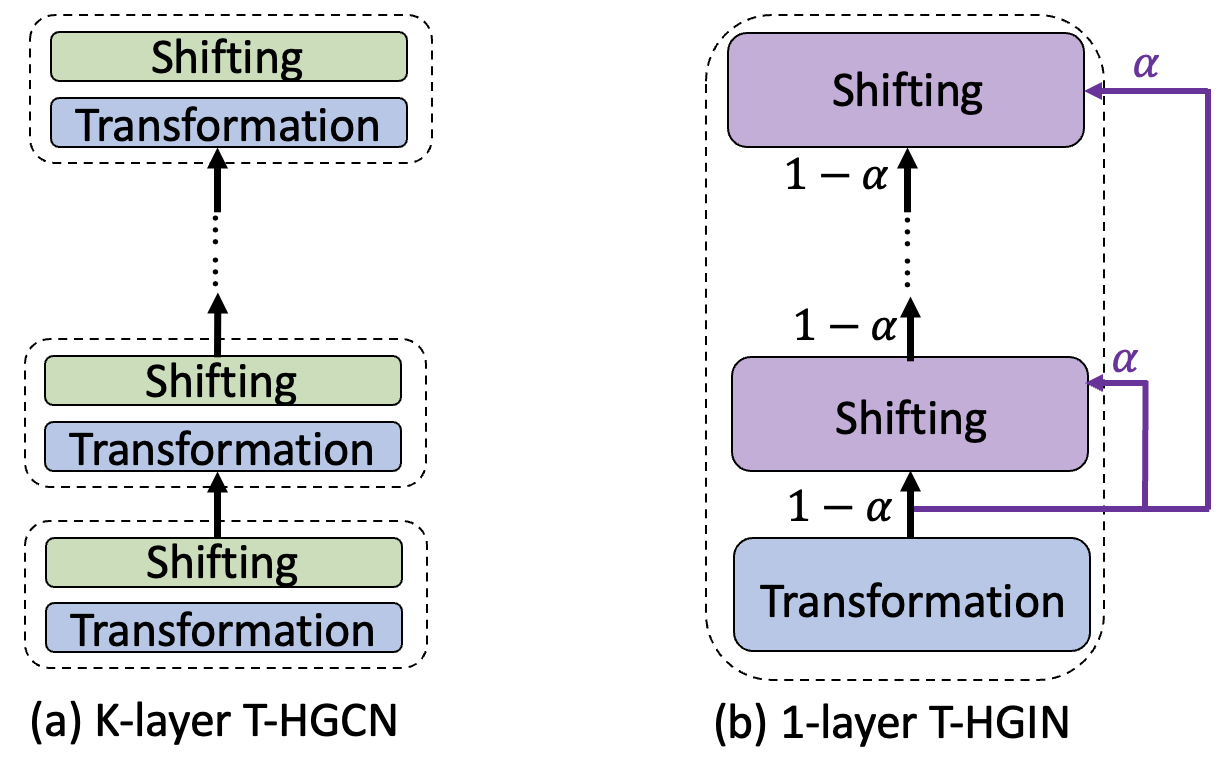}}
\caption{To perform $K$-step gradient descent for the underlying hypergraph signal denoising problem, we need (a) K-layer T-HGCN or alternatively (b) 1-layer T-HGIN.}\label{fig:modelcompare}
\vspace{-0.6cm}
\end{figure}

\section{Preliminaries}
\subsection{Hypergraph tensor representations and signal shifting}
While a hypergraph can be represented in either a matrix or a tensor form, in this work, we use tensorial descriptors to represent hypergraphs as they preserve intrinsic higher-order characteristics of hypergraphs~\cite{pena2022t}. Given a hypergraph $\mathcal{G} = (\mathcal{V}, \mathcal{E})$ containing $N$ nodes with order $M$ (that is, $m.c.e(\mathcal{G}) = M$), we define its {\bf normalized adjacency tensor} as an $M$-order $N$-dimensional tensor $\mathcal{A} \in \mathbb{R}^{N^M}$. Specifically, for any hyperedge $e_k = \{v_{k_1}, v_{k_2},...,v_{k_c} \} \in \mathcal{E}$ with $c = |e_k| \leq M$, the tensor's corresponding entries are given by
\begin{equation}\label{eq:adjacency}
     a_{p_1p_2...p_M} =\frac{1}{d(v_{p_1})} \frac{c}{\alpha},
\end{equation}
with
\begin{equation}
  \alpha = \sum_{r_1, r_2, ..., r_c \geq 1,\, \sum_{i=1}^c r_i = M} \binom{M}{r_1, r_2,\cdots, r_c},
  \label{eq:adjacency_alpha}
\end{equation}
and $d(v_{p_1})$ being the degree of node $v_{p_1}$ (or the total number of hyperedges containing $v_{p_1}$). The indices $p_1, p_2, ..., p_M$ for adjacency entries are chosen from all possible ways of $\{k_1, k_2, ..., k_c\}$'s permutations with at least one appearance for each element of the hyperedge set, and $\alpha$ is the sum of multinomial coefficients with the additional constraint $r_1, r_2,...,r_c \neq 0$. In addition, other entries not associated with any hyperedge are all zeros. Note that for any node $v_{p_1}\in \mathcal{V}$, we have $\sum_{p_2,...,p_M=1}^N a_{p_1p_2...p_M} =1$. 

The {\bf hypergraph signal tensor}, on the other hand, is designed as the $(M-1)$-time outer product of features along each feature dimension. Given the feature (or signal) matrix $\mathbf{X}\in \mathbb{R}^{N \times D}$ as the input, with $D$ being the dimension of features for each node, the $d$-th dimensional hypergraph signal ($d=1,\cdots,D$) is given by
\begin{equation}
    [\mathcal{X}]_d= \underbrace{[\mathbf{x}]_d \circ [\mathbf{x}]_d \circ \cdots \circ [\mathbf{x}]_d}_{\text{(M-1) times}} \in \mathbb{R}^{N\times 1 \times N^{(M-2)}},
    \label{eq:cni}
\end{equation}
where $\circ$ denotes the outer (elementary tensor) product, and $[\mathbf{x}]_d \in \mathbb{R}^N$ represents the $d$-th dimensional feature vector of all $N$ nodes. For example, given $M=3$, $ [\mathcal{X}]_d=[\mathbf{x}]_d [\mathbf{x}]_d^T \in \mathbb{R}^{N\times 1 \times N}$, where we unsqueeze the outer-product tensor to generate the additional second mode for the dimension index of different features. Then by computing $[\mathcal{X}]_d$ for all $D$ features and stacking them together along the second-order dimension, we obtain an $M^{\mathrm{th}}$-order interaction tensor $\mathcal{X}\in \mathbb{R}^{N\times D\times N^{(M-2)}}$. The resulting interaction tensor can be viewed as a collection of $D$ tensors, each depicting node interactions at one feature dimension. 
\par
Analogous to the simple graph signal shifting, \textbf{hypergraph signal shifting} is defined as the product of a hypergraph representation tensor $\mathcal{A}$ and a hypergraph signal tensor $\mathcal{X}$, offering the notion of information flow over a hypergraph. The tensor-tensor multiplications (known as t-products), in particular, preserve the intrinsic higher-order properties and are utilized to operate hypergraph signal shifting~\cite{pena2022t}. Take $M=3$ as a convenient example of the t-product. To provide an appropriate alignment in the t-product signal shifting (to be introduced in Eq.~\eqref{eq:tproduct}), we first symmetrize the adjacency tensor $\mathcal{A}\in \mathbb{R}^{N\times N\times N}$ to be $\mathcal{A}s \in \mathbb{R}^{N\times N\times (2N+1)}$ by adding a zero matrix $\mathbf{0}{N\times N}$ as the first frontal slice, reflecting the frontal slice of the underlying tensor, and then dividing by 2: $\mathcal{A}_s =\frac{1}{2}$ $\texttt{fold}([\mathbf{0}, \mathbf{A}^{(1)}, \mathbf{A}^{(2)}, ..., \mathbf{A}^{(N)}, \mathbf{A}^{(N)}, ..., \mathbf{A}^{(2)} ,\mathbf{A}^{(1)}])$, where the $k$-th frontal slice is $\mathbf{A}^{(k)} = \mathcal{A}(:,:,k)\in \mathbb{R}^{N\times N \times 1}$. After applying the same operation to the hypergraph tensor $\mathcal{X}$ and obtain $\mathcal{X}_s$, the hypergraph signal shifting is then defined through the t-product $*$ as
\begin{align}
&\quad\, \mathcal{A}_s*\mathcal{X}_s \\
&=\, \mathtt{fold}(\mathtt{bcirc}(\mathcal{A}_s)\cdot \mathtt{unfold}(\mathcal{X}_s))\\
    &=\, \mathtt{fold} \left(\begin{bmatrix}
                      \mathbf{0} & \mathbf{A}^{(1)} & \mathbf{A}^{(2)}  & \cdots & \mathbf{A}^{(2)} & \mathbf{A}^{(1)} \\
                     \mathbf{A}^{(1)} & \mathbf{0} & \mathbf{A}^{(1)}  & \cdots & \mathbf{A}^{(3)} & \mathbf{A}^{(2)}\\
                     \mathbf{A}^{(2)} & \mathbf{A}^{(1)} & \mathbf{0} & \cdots &\mathbf{A}^{(4)} & \mathbf{A}^{(3)}\\
                     \vdots & \vdots & \vdots & \ddots & \vdots & \vdots \\
                      \mathbf{A}^{(2)} &  \mathbf{A}^{(3)} &  \mathbf{A}^{(4)} & \cdots & \mathbf{0} & \mathbf{A}^{(1)}\\
                     \mathbf{A}^{(1)}  & \mathbf{A}^{(2)}& \mathbf{A}^{(3)} & \cdots &  \mathbf{A}^{(1)}  & \mathbf{0} 
                    \end{bmatrix}
                    \begin{bmatrix}
                    \mathbf{0}\\
                    \mathbf{X}^{(1)}\\
                    \mathbf{X}^{(2)}\\
                    \vdots\\
                    \mathbf{X}^{(2)}\\
                    \mathbf{X}^{(1)}\\
                    \end{bmatrix}\right),
	\label{eq:tproduct}
	\end{align}
where $\mathtt{bcirc}(\mathcal{A}_s)$ converts the set of $N_s$ frontal slice matrices (in $\mathbb R^{N\times N}$) of the tensor $\mathcal{A}_s$ into a block circulant matrix. The $\mathtt{unfold}(\mathcal{X}_s)$ stacks vertically the set of $N_s$ frontal slice matrices (in $\mathbb R^{N\times D}$) of $\mathcal{X}_s$ into a $N_sN \times D$ matrix. The $\mathtt{fold}()$ is the reverse of the $\mathtt{unfold}()$ process so that $\mathtt{fold}(\mathtt{unfold}(\mathcal{A}_s))=\mathcal{A}_s$. The t-product of higher order tensors is more involved with recursive computation with $3^{\mathrm{rd}}$ order base cases. To maintain presentation brevity here, a reader may refer to literature~\cite{order-p_tensor} for full technical details of the t-product $*$.

\subsection{Tensor-Hypergraph Convolutional Neural Network}
With the defined hypergraph signal shifting operation, a single T-HGCN~\cite{t-hypergnn} layer is given by $\mathcal{Y}_s = \mathcal{A}_s*\mathcal{X}_s*\mathcal{W}_s$, where $\mathcal{W}_s \in \mathbb{R}^{D \times D' \times N_s^{(M-2)}}$ is a learnable weight tensor with $DD'$ weights parameterized in the first frontal slice and all the remaining frontal slices being zeros. Since the t-product is commutable~\cite{order-p_tensor}, we rewrite the T-HGCN into the following two steps:
\begin{align}\label{eq:T-HGCN}
 \begin{cases}
        &\text{Signal transformation: } \mathcal{X}'_s = \text{MLP}(\mathcal{X}_s); \\
        & \text{Signal shifting: } \mathcal{Y}_s = \mathcal{A}_s * \mathcal{X}'_s,
 \end{cases}
\end{align}
where $\mathcal{X}_s \in \mathbb{R}^{N\times D\times N_s^{(M-2)}}$ and $\mathcal{Y}_s \in \mathbb{R}^{N\times D'\times N_s^{(M-2)}}$ are the input and output of a T-HGCN layer. To perform downstream tasks, non-linear activation functions can be applied to $\mathcal{Y}_s$ accordingly.

\section{Equivalence Between T-HGCN and Tensor Hypergraph Signal Denoising}\label{sec:equivalency}
Recall that the signal-shifting function $\phi_{shift}$ aggregates neighboring signals to infer the target signal of each node. The intuition behind the architecture of HyperGNNs (especially the signal shifting) is that connected nodes tend to share similar properties, that is, signals over a hypergraph are smooth. Motivated by this intuition and precious work~\cite{ma2021unified} on simple graphs, we introduce the tensor Hypergraph signal denoising (HyperGSD) problem with the smoothness regularization term and prove its equivalency to T-HGCN in this section.
\subsection{Tensor Hypergraph Signal Denoising}
\par
\textbf{Problem (Hypergraph Signal Denoising).} Suppose $\mathcal{X}_s \in \mathbb{R}^{N\times D \times N_s^{(M-2)}}$ is the hypergraph signal of an observed noisy hypergraph signal on an $M^{\text{th}}$ order hypergraph $\mathcal{G}$. Without loss of generality, we assume $D=1$ (if $D>1$, we can simply take summation over all feature dimensions and obtain the same result).  Motivated by a smoothness assumption of hypergraph signals, we formulate the HyperGSD problem with the Laplacian-based total variation regularization term as follows: 
\begin{equation}
\label{eq:denoising}
 \argmin_{\mathcal{Y}_s} \mathcal{J} = (\mathcal{Y}_s-\mathcal{X}_s)^T*(\mathcal{Y}_s-\mathcal{X}_s) + b  \mathcal{Y}_s^T * \mathcal{L}_s * \mathcal{Y}_s, 
\end{equation}
where $\mathcal{Y}_s \in \mathbb{R}^{N\times 1 \times N_s^{(M-2)}}$ is the desired hypergraph signal that we aim to recover, $b>0$ is a scalar for the regularization term, and the last $M-2$ orders of all the tensors are flattened as frontal slice indices to simplify the t-product. Here, $\mathcal{L}_s=\mathcal I_s-\mathcal{A}_s$ is the normalized symmetric Laplacian tensor, and $\mathcal I_s$ is an identity tensor (with the first frontal slice being identity matrix and the other entries being zero). The tensor transpose of $\mathcal{Y}_s\in \mathbb{R}^{N\times 1 \times N_s^{(M-2)}}$, under the t-algebra, is defined as $\mathcal{Y}^T_s\in \mathbb{R}^{1\times N \times N_s^{(M-2)}}$, which is obtained by recursively transposing each sub-order tensor and then reversing the order of these sub-order tensors~\cite{order-p_tensor}. The first term encourages the recovered signal $\mathcal{Y}_s$ to be close to the observed signal $\mathcal{X}_s$, while the second term encodes the regularization as neighboring hypergraph signals tend to be similar. Notice that the cost function $\mathcal{J}(\mathcal{Y}_s)$ is not a scalar, but a tensor in $1\times 1\times N_s^{(M-2)}$.
\subsection{T-HGCN as Hypergraph Signal Denoising}
Next, we show the key insight that the hypergraph signal shifting operation in the T-HGCN is directly connected to the HyperGSD problem, which is given in the following theorem.
\par
\begin{theorem}\label{theorem:signal_denoising}
The hypergraph signal shifting $\mathcal{Y}_s = \mathcal{A}_s * \mathcal{X}_s$ is equivalent to a one-step gradient descent of solving the leading function of the HyperGSD problem Eq.~\eqref{eq:denoising} with $c = \frac{1}{2b}$, where $c$ is the learning rate of the gradient descent step.
\end{theorem}
\begin{proof} First take the derivative of the cost function $\mathcal{J}(\mathcal{Y}_s)$ w.r.t $\mathcal{Y}_s$:
\begin{equation}
   \frac{\partial \mathcal{J}}{\partial \mathcal{Y}_s}= 2 \cdot \mathtt{bcirc}
   (\mathcal{Y}_s-\mathcal{X}_s) + 2b\cdot \mathtt{bcirc}(\mathcal{L}_s * \mathcal{Y}_s).
    \label{eq:grad}
\end{equation}

Recall from Eq.~\eqref{eq:tproduct} that the $\mathtt{bcirc}(\cdot)$ operation has the first column being the unfolded $2N+1$ frontal slices, and the other columns being the cyclic shifting of the first column. When updating $\mathcal{Y}_s$ using one-step gradient descent, the first column of a block circulant tensor is sufficient, as it contains all information of updating $\mathcal{Y}_s$, and the remaining columns differ from the first column in order only. Using the leading function $\mathcal{J}_1$ for Eq.~\eqref{eq:grad}, which gives the first block column of the circulant tensor $\frac{\partial \mathcal{J}}{\partial \mathcal{Y}_s}$, we can simply drop the $\text{bcirc}(\cdot)$ operation so that the one-step gradient descent to update $\mathcal{Y}_s$ from $\mathcal{X}_s$ is
\begin{align}
    \mathcal{Y}_s &\leftarrow \mathcal{X}_s - c \frac{\partial \mathcal{J}_1}{\partial \mathcal{Y}_s}\at[\Big]{\mathcal{Y} =  \mathcal{X}_s} \\
    &=  \mathcal{X}_s - 2bc (\mathcal{L}_s * \mathcal{X}_s)\\
    &= (1-2bc) \mathcal{X}_s + 2bc \mathcal{A}_s * \mathcal{X}_s.
\end{align}
Given learning rate $c = \frac{1}{2b}$, we obtain $\mathcal{Y}_s \leftarrow \mathcal{A}_s * \mathcal{X}_s$, which is the same form as the shifting operation in Eq.~\eqref{eq:T-HGCN}.
\end{proof}

This theorem implies that a single layer of T-HGCN~\cite{t-hypergnn} is essentially equivalent to solving the HyperGSD problem by one-step gradient descent. Correspondingly, performing a $K$-step gradient descent would require $K$ layers of T-HGCN, which could much increase the number of learnable parameters. As a result, a question naturally arises: Can we perform multi-step gradient descent toward the HyperGSD problem with just a single layer of HyperGNNs? We provide an affirmative answer by proposing the T-HGIN approach in the next section.

\section{Tensor-Hypergraph Iterative Network}
With the goal of merging multi-step gradient descent into a single HyperGNN, we first propose the $K$-step iterative gradient descent for the HyperGSD problem in Eq.~\eqref{eq:denoising}. Then we adopt the iteration process to design the Tensor-Hypergraph Iterative Network (T-HGIN).

{\bf Iterative Gradient Descent for Signal Denoising.}
Given the gradient of the HyperGSD problem in Eq.~\eqref{eq:grad}, we now update the gradient iteratively to obtain the sequence of hypergraph signals $(\mathcal{Y}_s^{(0)}, \mathcal{Y}_s^{(1)}, \mathcal{Y}_s^{(2)},..., \mathcal{Y}_s^{(K)} )$ with the following iterative process:
\begin{align}
    \mathcal{Y}_s^{(k)} &\leftarrow \mathcal{Y}_s^{(k-1)} - c \frac{\partial \mathcal{J}_1}{\partial \mathcal{Y}_s}\at[\Big]{\mathcal{Y}_s =  \mathcal{Y}_s^{(k-1)}} \notag\\
    =\,& (1-2b-2bc)  \mathcal{Y}_s^{(k-1)} + 2b\mathcal{X}_s + 2bc \mathcal{A}_s *  \mathcal{Y}_s^{(k-1)},\label{eq:iter_grad}
\end{align}
where $\mathcal{Y}_s^{(k)}$ with $k=1, ..., K$ are iteratively updated clean hypergraph signals and the starting point is $\mathcal{Y}_s^{(0)}=\mathcal{X}_s$.

{\bf From Iterative Signal Denoising To T-HGIN.}
From the updating rule above, we then formulate T-HGIN by a slight variation to Eq.~\eqref{eq:iter_grad}. Setting the regularization parameter $b = \frac{1}{2(1+c)}$, we then obtain that
\begin{equation}
    \mathcal{Y}_s^{(k)} \leftarrow 2b \mathcal{X}_s + 2bc \mathcal{A}_s * \mathcal{Y}_s^{(k-1)}.
\end{equation}
Since $2b + 2bc = 1$, setting $2b = \alpha$ implies that $2bc = 1-\alpha$. Consequently, a single layer of the T-HGIN is formulated as
\begin{align}\label{eq:T-HGIN}
\hspace{-.1in} \begin{cases}
        &\text{Signal transformation: } \mathcal{X}_s' = \text{MLP}(\mathcal{X}_s); \\
        &\text{Signal shifting: } \mathcal{Y}_s^{(k)} = \alpha \mathcal{X}_s' + (1-\alpha) \mathcal{A}_s * \mathcal{Y}_s^{(k-1)},
 \end{cases}
\end{align}
with $ k = 1,..., K$, $\mathcal{Y}_s^{(0)}=\mathcal{X}_s'$ and $\alpha \in [0, 1]$. The signal transformation is constructed by a MLP. The signal shifting of the T-HGIN can be roughly viewed as an iterative personalized PageRank~\cite{pagerank}, where $\alpha$ is the probability that a node will teleport back to the original node and $1-\alpha$ is the probability of taking a random walk on the hypergraph through the hypergraph signal shifting. In fact, when $\alpha=0$ and $K=1$, the T-HGIN is the same as the T-HGCN, indicating that the T-HGCN could be subsumed in the proposed T-HGIN framework. In addition, T-HGIN has three major advantages compared to T-HGCN: 
\begin{enumerate}
    \item As shown in Fig.~\ref{fig:modelcompare}, a $K$-layer T-HGCN is required to perform $K$ steps of hypergraph signal shifting, but in contrast, the T-HGIN breaks this required equivalence between the depth of neural networks and the steps of signal shifting, allowing any steps of signal shifting  in just one layer.
    \item The T-HGIN leverages the information contained in the original hypergraph signal $\mathcal{X}_s$, which performs a ``skip-connection" analogous to ResNet~\cite{resnet} and mitigates the potential over-smoothing problem~\cite{pagerank} as the neural network is going deep to aggregate broader neighborhood.
    \item Although the $K$-step hypergraph signal shifting is somewhat involved, the number of learnable parameters remains the same as only one layer of the T-HGCN. As shown in the following experiment, the T-HGIN can often achieve better performance than other alternative HyperGNNs that would require more learnable parameters.
\end{enumerate}

\section{Experiments}
We evaluate the proposed T-HGIN approach on three real-world academic networks and compare it to four state-of-the-art benchmarks. The experiment aims to conduct a semi-supervised node classification task, in which each node is an academic paper and each class is a research category. We use the accuracy rate to be the metric of model performance. For each reported accuracy rate, $50$ experiments are performed to compute the mean and the standard deviation of the accuracy rates. We use the Adam optimizer with a learning rate and the weight decay choosing from $\{0.01, 0.001\}$ and $\{0.005, 0.0005\}$ respectively, and tune the hidden dimensions over $\{64, 128, 256, 512\}$ for all the methods.

{\bf Datasets.} The hypergraph datasets we used are the co-citation datasets (Cora, CiteSeer, and PubMed) in the academic network. The hypergraph structure is obtained by viewing each co-citation relationship as a hyperedge. The node features associated with each paper are the bag-of-words representations summarized from the abstract of each paper, and the node labels are research categories (e.g., algorithm, computing, etc). For expedited proof of concept, the raw datasets from~\cite{hypergcn} are downsampled to smaller hypergraphs. The descriptive statistics of these hypergraphs are summarized in Table~\ref{table:data_stat1}.
\begin{table}[htp]
\vspace{-0.3cm}
\caption{Summary Statistics of the Academic Network Datasets}
\vspace{-0.5cm}
\begin{center}
    \begin{tabular}{c ccc}
\hline

 Statistics     & Cora   & Citeseer   & PubMed            \\
        \hline
$|\mathcal{V}|$   & 83     & 87         & 89             \\
$|\mathcal{E}|$ & 42     & 50         & 40             \\
Feature Dimension $D$   & 1433   & 3703       & 500      \\
 Number of Classes  & 7      & 6          & 3        \\
 Maximum Shortest Path & 2  & 4  & 3   \\
 Connected Components & 6 & 6 & 10 \\
 \hline
\end{tabular}\label{table:data_stat1}
\end{center}
\vspace{-0.4cm}
\end{table}

{\bf Experiment Setup and Benchmarks.} To classify the labels of testing nodes, we feed the whole hypergraph structure and node features to the model. The training, validation, and testing data are set to be $50\%, 25\%$, and $25\%$ for each complete dataset, respectively. We choose regular multi-layer perceptron (MLP), HGNN~\cite{HGNN}, HyperGCN~\cite{hypergcn}, and HNHN~\cite{hnhn} as the benchmarks. In particular, the HGNN and the HyperGCN utilize hypergraph reduction approaches to define the hypergraph adjacency matrix and Laplacian matrix, which may result in higher-order structural distortion~\cite{wan2021principled}. The HNHN formulates a two-stage propagation rule using the incidence matrix, which does not use higher-order interactions of the hypergraph signal tensor~\cite{t-hypergnn}. Following the convention of HyperGNNs, we set the number of layers for all HyperGNNs to be $2$ to avoid over-smoothing except for the T-HGCN and the proposed T-HGIN. For the T-HGCN and the T-HGIN, we use only one layer: the T-HGCN's accuracy decreases when the number of layers is greater than one, while the T-HGIN can achieve a deeper HyperGNN architecture by varying the times of iteration $K$ within one layer as shown in Fig.~\ref{fig:modelcompare} (b). The grid search is used to tune the two hyperparameters $K$ and $\alpha$ through four evenly spaced  intervals in both $K\in [1, 5]$ and $\alpha \in [0.1, 0.5]$. 

\begin{table}[htp]
\caption{Averaged testing accuracy ($\%$, $\pm$ standard deviation) on five academic networks for semi-supervised node classification. The top result for each dataset is highlighted.}
\begin{center}
    \begin{tabular}{l ccc}
    \hline\\[-0.6ex]
  Method          & Cora           & Citeseer     & PubMed       \\[0.5ex]
           \hline\\[-0.5ex]
MLP        & $48.23 \pm 7.35$     & $65.56\pm1.48 $ & $73.89 \pm 5.60 $ \\[0.5ex]
HGNN~\cite{HGNN}       & $70.59   \pm 1.22$ & $73.89\pm8.98$   & $82.22\pm1.33$  \\[0.5ex]
HyperGCN~\cite{hypergcn}   & $35.29 \pm 1.24 $  & $61.11\pm1.53$   & $76.11\pm1.40$     \\[0.5ex]
HNHN~\cite{hnhn}       &$ 69.41 \pm 9.04 $  & $74.44\pm9.69$   & $77.22 \pm 4.08$  \\[0.5ex]
T-HGCN~\cite{t-hypergnn} &$71.59 \pm 3.43$  &$ 78.33\pm8.03 $  & $86.67 \pm 1.18$  \\[0.5ex]
{\bf T-HGIN} (ours) & $\mathbf{73.64 \pm 2.47}$& $\mathbf{79.56 \pm 3.52}$& $\mathbf{90.31\pm 2.85}$\\[0.5ex]
\hline
\end{tabular}\label{table:exp1_result}
\vspace{-0.6cm}
\end{center}
\end{table}
{\bf Results and Discussion.} The averaged accuracy rates are summarized in Table~\ref{table:exp1_result}, which shows that our proposed $K$-step shifting entailed T-HGIN achieves the best performance among the state-of-the-art HyperGNNs on the three hypergraphs. While high variances of the results often occur to other existing HyperGNNs in these data examples, the proposed T-HGIN desirably shows only relatively moderate variance. 

{\bf The effect of the number of iterations.} Interestingly, the optimal values selected for $K$ coincide with the maximum shortest path on the underlying hypergraphs, the observation of which is consistent with that of~\cite{pagerank}. To some extent, this phenomenon supports the advantage of the proposed T-HGIN over other ``shallow" HyperGNNs that perform only one or two steps of signal shifting. Equipped with the multi-step iteration and the skip-connection mechanism, the T-HGIN is able to fully propagate across the whole hypergraph, and importantly, avoid the over-smoothing issue at the same time. 

{\bf The effect of the teleport probability.} Regarding the teleport parameter $\alpha$, the optimal selected values for the three datasets are $\{0.1, 0.1, 0.3\}$, respectively. 
Empirically, the selection of $\alpha$'s could depend on the connectivity of nodes. For example, the PubMed hypergraph has more isolated connected components and tends to require a higher value of $\alpha$. A direct visualization for the PubMed network is also shown in Fig.~\ref{fig: pubmed_visual} using one representative run of the experiment, which shows that the tensor-based approaches appear to give more satisfactory performance than the classic matrix-based HyperGNN; the proposed T-HGIN further improves upon the T-HGCN, confirming the effectiveness of the proposed multi-step iteration scheme.

\begin{figure}[htbp]
\vspace{-0.3cm}
\centerline{\includegraphics[width=0.5\textwidth]{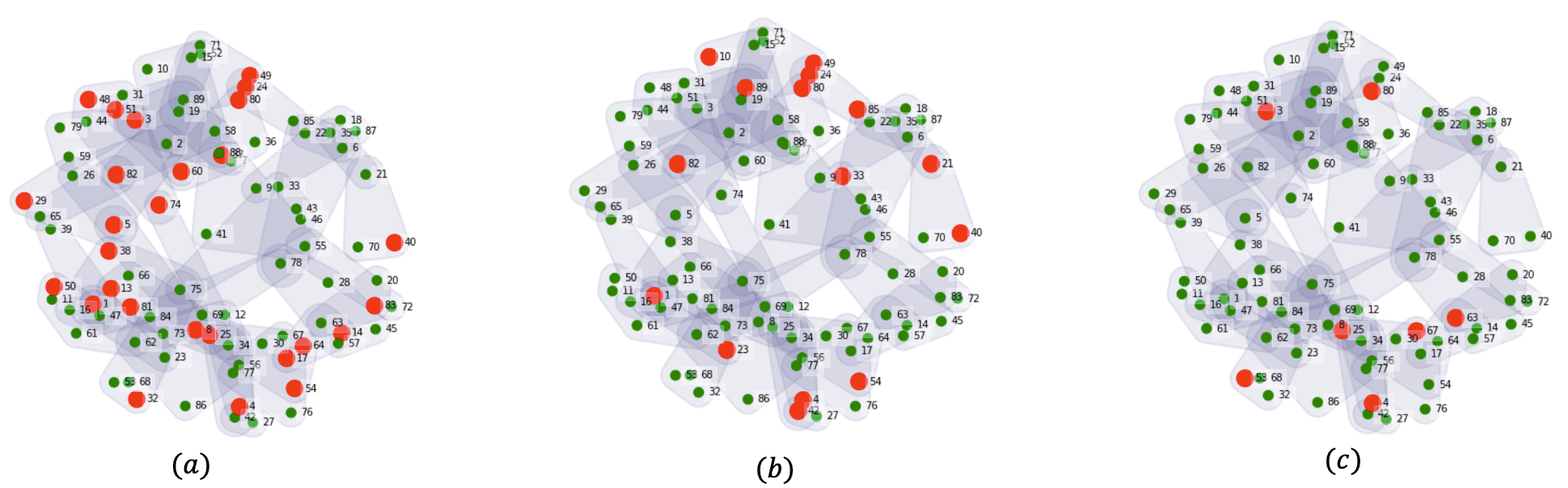}}
\caption{Comparison of label prediction on the PubMed dataset between (a) HyperGCN~\cite{hypergcn}, (b) T-HGCN~\cite{t-hypergnn}, and (c) T-HGIN. Red and green dots represent incorrectly and correctly classified nodes respectively.}
\label{fig: pubmed_visual}
\end{figure}

\section{Conclusion}
In the context of Tensor-HyperGraph Neural Networks (T-HyperGNNs), this work demonstrates that the hypergraph signal shifting of T-HGCN is equivalent to a one-step gradient descent of solving the hypergraph signal denoising problem. Based on this equivalency, we propose a $K$-step gradient descent rule and formulate a new hypergraph neural network -- Tensor-Hypergraph Iterative Network (T-HGIN). Compared to the T-HGCN, the T-HGIN benefits from the construction of $K$-step propagation in one single layer, offering an efficient way to perform propagation that spreads out to a larger-sized neighborhood. Satisfactorily, the proposed T-HGIN achieves competitive performance on multiple hypergraph data examples, showing its promising potential in real-world applications. We also note that the equivalency between HyperGNNs and HyperGSDs can also be utilized to design neural networks for denoising like in~\cite{rey2022untrained, rey2021overparametrized}, and we will leave this as an interesting extension for future studies.

\ifCLASSOPTIONcompsoc
  \section*{Acknowledgments}
\else
  \section*{Acknowledgment}
\fi
This work was partially supported by the NSF under grants CCF-2230161, DMS-1916376, the AFOSR award FA9550-22-1-0362, and by the Institute of Financial Services Analytics.

\bibliographystyle{IEEEtran}
\bibliography{mybib} 


\end{document}